\documentclass{new_tlp}
\usepackage[latin1]{inputenc}
\usepackage[]{todonotes}
\usepackage{amssymb}
\usepackage[fleqn]{amsmath}
\usepackage[all]{xy}
\usepackage{url}
\usepackage{amsfonts}
\usepackage{mdframed}
\usepackage{url}
\usepackage{hyperref}
\usepackage{algorithm}
\usepackage{algpseudocode}
\usepackage{xspace}
\usepackage{multirow}

\newcommand\cA{{\cal A}}

\newcommand\cC{{\cal C}}

\newcommand\cF{{\cal F}}

\newcommand\cI{{\cal I}}

\newcommand\cK{{\cal K}}
\newcommand\cL{{\cal L}}

\newcommand\cR{{\cal R}}

\newcommand\cT{{\cal T}}

\usepackage{xspace}

\newcommand{\shi}{$\mathcal{SHI}$\xspace}

\newcommand\fluffy{{\mathit{fluffy}}}

\newcommand\trillp{{TRILL$^P$}\xspace}
\newcommand\trillpbdd{{TORNADO}\xspace}

\newtheorem{theorem}{Theorem}
\newtheorem{definition}{Definition} 
\newtheorem{example}{Example} 

\title{Probabilistic DL Reasoning with Pinpointing Formulas: A Prolog-based Approach}
\author[R. Zese, G. Cota, E. Lamma, E. Bellodi,  and F. Riguzzi] 
{RICCARDO ZESE$^1$,  GIUSEPPE COTA$^1$, EVELINA LAMMA$^1$ \and ELENA BELLODI$^2$, FABRIZIO RIGUZZI$^2$\\
$^1$ Dipartimento di Ingegneria -- Universit\`a di Ferrara\\
Via Saragat 1, 44122, Ferrara, Italy  \\
$^2$ Dipartimento di Matematica e Informatica -- Universit\`a di Ferrara\\
Via Saragat 1, 44122, Ferrara, Italy\\
\email{name.surname@unife.it}}

\begin{document}
\maketitle

\begin{abstract}
  
  When modeling real world domains we have to deal with information that is incomplete or that comes from  sources with different trust levels. 
  This motivates the need for managing uncertainty in the Semantic Web. 
  To this purpose, we introduced a probabilistic semantics, named DISPONTE, in order to combine description logics with probability theory.
 The probability of a query can be then computed from the set of its explanations by  building a Binary Decision Diagram (BDD). 
  The set of explanations can be found using the \emph{tableau algorithm}, which has to handle non-determinism.
  Prolog, with its efficient handling of non-determinism, is  suitable  for implementing the tableau algorithm. TRILL and \trillp  are systems offering a Prolog implementation of the tableau algorithm.
  \trillp builds a \emph{pinpointing formula}, that compactly represents the set of explanations and can be directly translated into a BDD. Both reasoners were shown  to  outperform state-of-the-art DL reasoners.
  In this paper, we present an improvement of \trillp, named  \trillpbdd, in which the BDD is directly built during the construction of the tableau, further speeding up the overall inference process.
  An experimental comparison shows the effectiveness of \trillpbdd.
  All systems can be tried online in the TRILL on SWISH web application at \url{http://trill.ml.unife.it/}.
\end{abstract}

\section{Introduction}
\label{sec:intro}

The objective of the Semantic Web is to make information available in a form that is understandable and automatically manageable by machines.
In order to realize this vision, the W3C  supported the development of a family of 
knowledge representation formalisms of increasing complexity for defining ontologies, called Web Ontology Languages (OWL), based on Description Logics (DLs). 
In order to fully support the development of the Semantic Web, efficient DL reasoners are essential.
Usually, the most common approach adopted by reasoners is the \emph{tableau algorithm}~\cite{DBLP:journals/jar/HorrocksS07}, written in a procedural language. This algorithm applies some expansion rules on a tableau, a representation of the assertional part of the KB.
However, some of these rules are non-deterministic, requiring the implementation of a search strategy in an or-branching search space. Pellet~\cite{DBLP:journals/ws/SirinPGKK07}, for instance, is a reasoner written in Java.

Modeling real world domains requires dealing with information that is incomplete or that comes from 
sources with different trust levels. This motivated the need for managing uncertainty in the Semantic Web, and  led to many proposals for combining probability theory with OWL languages, or with the underlying DLs, such as P-$\mathcal{SHIQ}(\mathbf{D})$ \cite{DBLP:journals/ai/Lukasiewicz08}, $\mathcal{BEL}$~\cite{DBLP:conf/sum/CeylanP15}, Prob-$\mathcal{ALC}$~\cite{DBLP:conf/kr/LutzS10}, PR-OWL~\cite{DBLP:conf/semweb/CarvalhoLC13}, and those proposed in \citeN{DBLP:conf/semweb/JungL12}, \citeN{DBLP:conf/uai/Heinsohn94}, \citeN{DBLP:conf/kr/Jaeger94}, \citeN{DBLP:conf/aaai/KollerLP97}, \citeN{Ding04aprobabilistic}.

In \cite{bellodi2011distribution,RigBelLamZes15-SW-IJ,Zese17-SSW-BK}  we introduced DISPONTE, a probabilistic semantics for DLs. 
DISPONTE follows the distribution semantics \cite{DBLP:conf/iclp/Sato95}  derived from Probabilistic Logic Programming (PLP), that has emerged as one of the most effective approaches 
for representing probabilistic information in Logic Programming languages.
Many techniques have been proposed in PLP for combining Logic Programming with probability theory, for example \cite{DBLP:journals/tplp/LakshmananS01} and \cite{DBLP:journals/jlp/KiferS92} defined an extended immediate consequence operator that deals with probability intervals associated with atoms, effectively propagating the uncertainty among atoms using rules.


Despite the number of proposals for probabilistic semantics extending DLs, only few of them have been equipped with a reasoner to compute the probability of queries. 
Examples of probabilistic DL reasoners are PRON\-TO \cite{DBLP:conf/esws/Klinov08}, BORN~\cite{DBLP:conf/ore/CeylanMP15} and BUNDLE~\cite{RigBelLamZes15-SW-IJ,Zese17-SSW-BK}.
PRONTO, for instance, is a probabilistic reasoner that can be applied to P-$\mathcal{SHIQ}(\mathbf{D})$.
BORN answers probabilistic subsumption queries w.r.t. $\mathcal{BEL}$ KBs by using ProbLog for managing the probabilistic part of the KB.
Finally, BUNDLE performs probabilistic reasoning over DISPONTE KBs by exploiting Pellet to return explanations and Binary Decision Diagrams (BDDs) to compute the probability of queries. 



Usually DL reasoners adopt the \emph{tableau algorithm}~\cite{DBLP:journals/jar/HorrocksS07,horrocks2006even}. This algorithm applies some expansion rules on a tableau, a representation of the assertional part of the KB.
However, some of these rules are non-deterministic, requiring the implementation of a search strategy in an or-branching search space. 


Reasoners written in Prolog can exploit Prolog's backtracking facilities for performing the search, as has been observed in various works \cite{DBLP:journals/jar/BeckertP95,DBLP:dblp_journals/iandc/HustadtMS08,DBLP:dblp_journals/tplp/LukacsyS09,DBLP:journals/logcom/RiccaGSDGL09,DBLP:conf/iclp/GavanelliLRBZC15}.
For this reason, in \cite{ZesBelRig16-AMAI-IJ,Zese17-SSW-BK} we proposed  the system TRILL, 
a tableau reasoner implemented in Prolog. 
Prolog's search strategy is exploited for taking into account the non-determinism of the tableau rules.
TRILL can check the consistency of a concept and the entailment of an axiom from an ontology, and can also return the probability of a query.

Both BUNDLE and TRILL use Binary Decision Diagrams (BDDs) for computing the probability of queries from the set of all explanations.
They encode the results of the inference process in a BDD from which the probability can be computed in a time linear in the size of the diagram. 
We also developed \trillp~\cite{ZesBelRig16-AMAI-IJ,Zese17-SSW-BK}, which builds a \emph{pinpointing formula} able to compactly represent the set of explanations.
This formula is used to build the corresponding BDD and compute the query's probability.
In \cite{RigBelLamZes15-SW-IJ,ZesBelRig16-AMAI-IJ,Zese17-SSW-BK} we have extensively tested BUNDLE, TRILL and \trillp, showing that they can achieve significant results in terms of scalability and speed.

In this paper, we present \trillpbdd for ``Trill powered by pinpOinting foRmulas and biNAry DecisiOn diagrams'', in which the BDD representing the pinpointing formula is directly built during tableau expansion, speeding up the overall inference process.
TRILL, \trillp and \trillpbdd are all available in the TRILL on SWISH web application at \url{http://trill.ml.unife.it/}.

We also present an experimental evaluation of \trillpbdd by comparing it with several probabilistic and non-probabilistic reasoners. Results show that \trillpbdd is as fast as or faster than state-of-art reasoners also for non-probabilistic inference and can, in some cases, avoid an exponential blow-up.

The paper is organized as follows: Section~\ref{sec:description-logics} briefly introduces DLs and Section~\ref{sec:prob-descr-logics} presents DISPONTE. 
The tableau algorithm of \trillp and \trillpbdd is discussed in Section~\ref{sec:reasoning}, followed by the description of the two systems in Section~\ref{sec:systems}.
Finally, Section~\ref{sec:exp} shows the experimental evaluation and Section~\ref{sec:concl} concludes the paper.

\section{Description Logics}
\label{sec:description-logics}

DLs are fragments of FOL languages used for modeling knowledge bases (KBs) that exhibit nice computational properties such as decidability and/or low complexity \cite{Badeer:2008:DL:52211}. 
There are many  DL languages that differ by the constructs that are allowed for defining concepts (sets of individuals of the domain) and roles (sets of pairs of individuals). Here we illustrate the DL \shi which is the expressiveness level supported by \trillp and \trillpbdd.

Let us consider a set of \emph{atomic concepts} $\mathbf{C}$, a set of \emph{atomic roles} $\mathbf{R}$ and a set of individuals $\mathbf{I}$.
A \emph{role} could be an atomic role $R \in \mathbf{R}$ or the inverse $R^{-}$ of an atomic role $R \in \mathbf{R}$. We use $\mathbf{R^{-}}$ to denote the set of all inverses of roles in $\mathbf{R}$. 
Each $A \in \mathbf{A}$, $\bot$ and $\top$ are concepts. 
If $C$, $C_1$ and $C_2$ are concepts and $R \in \mathbf{R}\cup\mathbf{R^{-}}$, then $(C_1\sqcap C_2)$, $(C_1\sqcup C_2 )$ and $\neg C$ are concepts, as
well as $\exists R.C$ and $\forall R.C$.

A \emph{knowledge base} (KB) $\cK = (\cT, \cR, \cA)$ consists of a TBox $\cT$, an RBox $\cR$ and an ABox $\cA$. An RBox $\cR$ is a finite set of \emph{transitivity axioms} $Trans(R)$ and \emph{role inclusion axioms} $R \sqsubseteq S$, where $R, S \in \mathbf{R} \cup \mathbf{R^{-}}$.
A \emph{TBox} $ \cT $ is a finite set of \textit{concept inclusion axioms} $C\sqsubseteq D$, where $C$ and $D$ are concepts.
An \emph{ABox} $\cA$ is a finite set of \textit{concept membership axioms} $a : C$ and \textit{role membership
axioms} $(a, b) : R$, where $C$ is a concept, $R \in \mathbf{R}$ and $a,b \in \mathbf{I}$. 

A \shi KB is usually assigned a semantics in terms of interpretations $\cI = (\Delta^\cI , \cdot^\cI )$, where $\Delta^\cI$ is a non-empty \textit{domain} and $\cdot^\cI$ is the \textit{interpretation function}, which assigns  an element in $\Delta ^\cI$ to each $a \in \mathbf{I}$, a subset of $\Delta^\cI$ to each concept and a subset of $\Delta^\cI \times \Delta^\cI$ to each role.

A query $Q$ over a KB $\cK$ is usually an axiom for which we want to test the entailment from the KB, written as $\cK \models Q$. 

\begin{example}
\label{people+pets}
The following KB is inspired by the ontology  \texttt{people+pets} \cite{ISWC03-tut}:
$$\begin{array}{lcl}
\exists hasAnimal.Pet \sqsubseteq NatureLover&\ &Cat\sqsubseteq Pet\\
\fluffy: Cat &\ & (kevin,\fluffy):hasAnimal  \\
tom: Cat &\ & (kevin,tom):hasAnimal
\end{array}$$
It states that  individuals that own an animal which is a pet are nature lovers and that $kevin$ owns the animals $\fluffy$ and $tom$, which are cats. Moreover,  cats are pets.
The KB entails the query $Q=kevin:NatureLover$.
\end{example}

\section{Probabilistic Description Logics}
\label{sec:prob-descr-logics}
DISPONTE \cite{bellodi2011distribution,RigBelLamZes15-SW-IJ,Zese17-SSW-BK} applies the distribution semantics to probabilistic ontologies~\cite{DBLP:conf/iclp/Sato95}. 
In DISPONTE a \emph{probabilistic knowledge base} $\cK$ is a set of certain and probabilistic axioms.
\emph{Certain axioms} are regular DL axioms.  
 \emph{Probabilistic axioms} take the form $p:: E$, where $p$ is a real number in $[0,1]$ and $E$ is a DL axiom.
Probability $p$ can be interpreted as the degree of our belief in axiom $E$.
For example, a probabilistic concept membership axiom $p::a:C$ means that we have degree of belief $p$ in $a : C$. The statement that cats are pets with probability 0.6 can be expressed as
$0.6::Cat\sqsubseteq Pet$.
 
The idea of DISPONTE is to associate independent Boolean random variables with the probabilistic axioms. By assigning values to every random variable we obtain a \emph{world}, i.e. the set of probabilistic axioms whose random variable takes on value 1 together with the set of certain axioms.
Therefore, given a KB with $n$ probabilistic axioms, there are $2^n$ different worlds, one for each possible subset of the probabilistic axioms. Each world contains all the non-probabilistic axioms of the KB.
DISPONTE defines a probability distribution over worlds as in probabilistic logic programming.

The probability of a world $w$ is computed by multiplying the probability $p$ for each probabilistic axiom included in the world with the probability $1-p$ for each probabilistic axiom not included in the world.

Formally, an \emph{atomic choice} is a couple $(E_i,k)$ where $E_i$ is the $i$-th probabilistic axiom  and $k\in \{0,1\}$. 
$k$ indicates whether $E_i$ is chosen to be included in a world ($k$ = 1) or not ($k$ = 0). 
A \emph{composite choice} $\kappa$ is a consistent set of atomic choices, i.e.,  $(E_i,k)\in\kappa, (E_i,m)\in \kappa$ implies $k=m$ (only one decision is taken for each axiom). 
The probability of a composite choice $\kappa$  is 
$P(\kappa)=\prod_{(E_i,1)\in \kappa}p_i\prod_{(E_i, 0)\in \kappa} (1 - p_i)$, where $p_i$ is the probability associated with axiom $E_i$.
A \emph{selection} $\sigma$ is a total composite choice, i.e., it contains an atomic choice $(E_i,k)$ for every 
probabilistic axiom  of the theory. 
Thus a selection $\sigma$ identifies a \emph{world} in this way:
$w_\sigma=\cC\cup\{E_i|(E_i,1)\in \sigma\}$ where $\cC$ is the set of certain axioms. 
Let us indicate  with $\mathcal{W}_\cK$ the set of all worlds.
The probability of a world $w_\sigma$  is 
$P(w_\sigma)=P(\sigma)=\prod_{(E_i,1)\in \sigma}p_i\prod_{(E_i, 0)\in \sigma} (1-p_i)$.
$P(w_\sigma)$ is a probability distribution over worlds, i.e., $\sum_{w\in \mathcal{W}_\cK}P(w)=1$.

We can now assign probabilities to queries.
Given a world $w$ the probability of a query $Q$ is defined as $P(Q|w)=1$ if $w\models Q$ and 0 otherwise. The probability of a query can be obtained by marginalizing the joint probability of the query and the worlds $P(Q,w)$:
\begin{eqnarray}
P(Q)&=&\sum_{w\in \mathcal{W}_\cK}P(Q,w)\label{pq}\\
&=&\sum_{w\in \mathcal{W}_\cK} P(Q|w)P(w)\label{pq1}\\
&=&\sum_{w\in \mathcal{W}_\cK: w\models Q}P(w)\label{pq2}
\end{eqnarray}

\pagebreak
\begin{example} \label{people+pets2}
 Let us consider the knowledge base and the query $Q=kevin:natureLover$ of Example~\ref{people+pets} where some of the axioms are made probabilistic:
$$\begin{array}{clccl}
 (C_1)&\exists hasAnimal.Pet \sqsubseteq NatureLover &\ & (E_1)&0.4\ ::\ \fluffy: Cat\\
 (C_2)& (kevin,\fluffy):hasAnimal &\ & (E_2)&0.3\ ::\ tom: Cat\\
 (C_3)& (kevin,tom):hasAnimal &\ & (E_3)&0.6\ ::\ Cat\sqsubseteq Pet
 \end{array}$$
\noindent
 $\fluffy$ and $tom$ are cats and  cats are pets with the specified probabilities.
The KB has eight worlds and $Q$ is true in three of them, i.e.,
$$\{C_1, C_2, C_3, E_1, E_3\},\{C_1, C_2, C_3, E_2, E_3\},\{C_1, C_2, C_3, E_1, E_2, E_3\}.$$
These worlds corresponds to the selections:
$$\{(E_1,1),(E_2,0),(E_3,1)\},\{(E_1,0),(E_2,1),(E_3,1)\},\{(E_1,1),(E_2,1),(E_3,1)\}.$$
The probability is 
$P(Q) = 0.4\cdot 0.7\cdot 0.6 + 0.6\cdot0.3\cdot0.6+0.4\cdot 0.3\cdot0.6 = 0.348.$
\end{example}

\noindent 
TRILL \cite{ZesBelRig16-AMAI-IJ,Zese17-SSW-BK}
computes the probability of a query w.r.t. KBs that follow DISPONTE by first computing all the explanations for the query and then building a Binary Decision Diagram (BDD) that represents them.
An explanation is a subset of axioms $\kappa$ of a KB $\cK$ such that $\kappa \models Q$. Since explanations may contain also axioms that are irrelevant for proving the truth of $Q$, usually, minimal explanations\footnote{Also known as \emph{justifications}.} w.r.t. set inclusion are considered. This means that a set of axioms $\kappa\subseteq \cK$ is a minimal explanation if $\kappa\models Q$ and for all $\kappa' \subset \kappa$, $\kappa' \not\models Q$, i.e. $\kappa'$ is not an explanation for $Q$.
Therefore, consider $\kappa$ a minimal explanation, if we remove one of the axioms in $\kappa$, creating the set $\kappa'$, then $\kappa'$ is not an explanation, while if we add an axiom randomly chosen among those contained in the KB to $\kappa$, creating $\kappa''$, then $\kappa''$ is an explanation that is not minimal. From now on, we will consider only minimal explanations. For the sake of brevity, when we will mention explanations we will refer to minimal explanations.
An explanation can be represented with a composite choice. 
Given the set $K$ of all explanations for a query $Q$, we can define the  Disjunctive Normal Form (DNF) Boolean formula $f_K$ as
$f_K(\mathbf{X})=\bigvee_{\kappa\in K}\bigwedge_{(E_i,1)}X_{i}$.
The variables $\mathbf{X}=\{ X_i|p_i::E_i \in \cK\}$ are independent Boolean random variables with $P(X_i=1)=p_i$ and the probability  that $f_K(\mathbf{X})$ takes value 1 gives the probability of $Q$.
A BDD for a  function of Boolean variables is   
a rooted graph that has one level for each Boolean variable. 
A node $n$  has two children: one corresponding to the 1 value of the variable associated with the level of $n$
and one corresponding to the 0 value of the variable.
When drawing BDDs, the 0\--branch 
is distinguished from the 1-branch by drawing it with a dashed line.
The leaves store either 0 or 1.
BDD software packages take as input a Boolean function $f(\mathbf{X})$ and incrementally build the diagram so that isomorphic portions of it are merged, possibly changing the order of variables if useful. This often allows the diagram to have a number of nodes much smaller than exponential in the number of variables that a naive representation of the function would require.

Given the BDD, we can use the function \textsc{Prob} shown in Algorithm \ref{alg:prob} \cite{DBLP:journals/tplp/KimmigDRCR11}. This dynamic programming algorithm traverses the diagram from the leaves and computes the probability of a formula encoded as a BDD.
\begin{algorithm}[ht]
\caption{Function \textsc{Prob}: it takes a BDD encoding a formula and computes its probability.\label{alg:prob}} 
\begin{footnotesize}
\begin{algorithmic}[1]
\Function{Prob}{$node$, $nodesTab$}
\State Input: a BDD node $node$
\State Input: a table containing the probability of already visited nodes $nodesTab$
\State Output: the probability of the Boolean function associated with the node
\If{$node$ is a terminal}
\State return $value(node)$\Comment{$value(node)$ is 0 or 1}
\Else
\State scan $nodesTab$ looking for $node$
\If{found}
\State let $P(node)$ be the probability of $node$ in $nodesTab$
\State return $P(node)$
\Else
\State let $X$ be $v(node)$ \Comment{$v(node)$ is the variable associated with $node$}
\State $P_1\gets$\Call{Prob}{$child_1(node)$}
\State $P_0\gets$\Call{Prob}{$child_0(node)$}
\State $P(node)\gets P(X)\cdot P_1+(1-P(X))\cdot P_0 $
\State add the pair ($node$,$P(node)$) to $nodesTab$
\State return $P(node)$
\EndIf
\EndIf
\EndFunction
\end{algorithmic}
\end{footnotesize}
\end{algorithm}

\begin{example}[Example~\ref{people+pets2} cont.] 
\label{people+pets3}
 Let us consider the KB of Example~\ref{people+pets2}.
If we associate the random variables $X_{1}$ with axiom $E_1$, $X_{2}$ with $E_2$ and $X_3$ with $E_3$,
 the Boolean formula $f(\mathbf{X})=(X_{1}\wedge X_{3})\vee (X_{2}\wedge X_{3})$ represents the set of explanations. The BDD for such a function is shown in Figure \ref{dd}. 
By applying function \textsc{Prob} of Algorithm~\ref{alg:prob} to this BDD we get
\begin{eqnarray*}
\mbox{\textsc{Prob}}(n_3)&=&0.6\cdot 1+0.4\cdot 0=0.6\\
\mbox{\textsc{Prob}}(n_2)&=&0.4\cdot 0.6+0.6\cdot 0=0.24\\
\mbox{\textsc{Prob}}(n_1)&=&0.3\cdot 0.6+0.7\cdot 0.24=0.348
\end{eqnarray*}
and therefore $P(Q)=\mbox{\textsc{Prob}}(n_1)=0.348$, which corresponds to the probability given by the semantics.
\end{example}

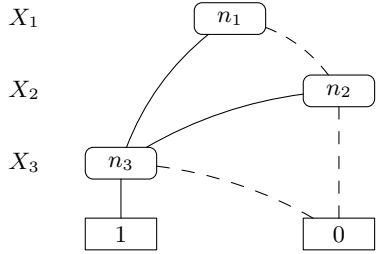
\begin{figure}
 $$
 \xymatrix@=5mm
 { X_{1} & &*=<27pt,12pt>[F-:<3pt>]{n_1}
 \ar@/_/@{-}[ldd] \ar@/^/@{--}[dr]\\ 
 X_{2}  & & &*=<27pt,12pt>[F-:<3pt>]{n_2} 
 \ar@/_/@{-}[dll]\ar@{--}[dd] 
 \\
 X_{3}& *=<27pt,12pt>[F-:<3pt>]{n_3}
 \ar@{-}[d] \ar@/^/@{--}[drr]  \\
 &*=<27pt,12pt>[F]{1} &&*=<27pt,12pt>[F]{0}}
 $$
 \caption{BDD representing the set of explanations for the query of Example  \ref{people+pets}. \label{dd}
}
 \end{figure}

\section{The Pinpointing Formula}
\label{sec:reasoning}
In \cite{DBLP:journals/jar/BaaderP10,DBLP:journals/logcom/BaaderP10} the authors consider the problem of finding a \emph{pinpointing formula} instead of a set of explanations. A pinpointing formula is a compact representation of the set of explanations.
To build a pinpointing formula, first we have to associate a unique propositional variable with every axiom $E$ of the KB $\cK$, indicated with $var(E)$. Let $var(\cK )$ be the set of all the propositional variables associated with axioms in $\cK$, then the pinpointing formula is a \emph{monotone Boolean formula} built using some or all of the variables in $var(\cK)$ and the conjunction and disjunction connectives. 
 A valuation $\nu$ of a set of variables $var(\cK)$ is the set of propositional variables that are true, i.e., $\nu \subseteq var(\cK)$. For a valuation $\nu \subseteq var(\cK)$, let $\cK_{\nu} := \{E \in \cK |var(E)\in\nu\}$.
 \begin{definition}[Pinpointing formula \cite{DBLP:journals/logcom/BaaderP10}]
 Given a query $Q$ and a KB $\cK$, a monotone Boolean formula $\phi$ over $var(\cK)$ is called a \emph{pinpointing formula} for $Q$ if for every valuation $\nu \subseteq var(\cK)$ it holds that $\cK_{\nu} \models Q$ iff
 $\nu$ satisfies $\phi$.
  
 \end{definition}
 In \cite{DBLP:journals/logcom/BaaderP10} the authors also discuss the relation between the pinpointing formula and explanations for a query $Q$.
 Let us denote the set of explanations for $Q$ by $Expls(\cK,Q) = \{\cK_{\nu} | \nu$ \textit{is a minimal valuation satisfying} $\phi\}$. $Expls(\cK,Q)$ can be obtained by converting the pinpointing formula into Disjunctive Normal Form (DNF) and removing disjuncts implying other disjuncts. However, the transformation to DNF may produce a formula whose size is exponential in the size of the original one. In addition, the correspondence holds also in the other direction: the formula
 $\bigvee_{Ex\in Expls(\cK,Q)} \bigwedge_{E\in Ex} var(E)$
 is a pinpointing formula.
 
 \begin{example}[Example~\ref{people+pets3} cont.]
 \label{people+pets4}
 Let us consider the KB $\cK$ and the query $Q$ of Example~\ref{people+pets2}. The set $Expls(\cK,Q) =\{\{C_2,E_1,E_3,C_1\},\{C_3,E_2,E_3,C_1\}\}$ corresponds to the pinpointing formula $(C_2\wedge E_1\wedge E_3\wedge C_1)\vee(C_3\wedge E_2\wedge E_3\wedge C_1)$.
  
 \end{example}
 One interesting feature of the pinpointing formula is that an exponential number of explanations can be represented with a much smaller pinpointing formula.
 \begin{example}
  \label{exp-expl}
  Given an integer $n\geq1$, consider the  KB containing the following axioms for $1\leq i\leq n$:
  $$\begin{array}{lcr}
     (C_{1,i})\ B_{i-1}\sqsubseteq P_i\sqcap Q_i&\ \ \ \ (C_{2,i})\ P_i\sqsubseteq B_i&\ \ \ \ (C_{3,i})\ Q_i\sqsubseteq B_i
    \end{array}$$
The query $Q= B_0\sqsubseteq B_n$ has $2^n$ explanations, even if the KB has a size that is linear in $n$.
For $n=2$ for example, we have 4 different explanations, namely
$$\begin{array}{l}
\{C_{1,1}, C_{2,1}, C_{1,2}, C_{2,2}\}\\
\{C_{1,1}, C_{3,1}, C_{1,2}, C_{2,2}\}\\
\{C_{1,1}, C_{2,1}, C_{1,2}, C_{3,2}\}\\
\{C_{1,1}, C_{3,1}, C_{1,2}, C_{3,2}\}
\end{array}$$
The corresponding pinpointing formula is $C_{1,1} \wedge (C_{2,1}\vee C_{3,1})\wedge C_{1,2}\wedge(C_{2,2}\vee C_{3,2})$. 
In general, given $n$, the formula for this example is
$$\bigwedge_{i\in\{1,n\}} C_{1,i} \wedge \bigwedge_{j\in\{1,n\}} \bigvee_{z\in\{2,3\}} C_{z,j}$$
\noindent
whose size is linear in $n$.
 \end{example}

  \subsection{The Tableau Algorithm for the Pinpointing Formula}
  
One of the most common approaches for performing inference in DL is the tableau algorithm~\cite{DBLP:journals/sLogica/BaaderS01}. A tableau is a graph where the nodes are individuals annotated with the concepts they belong to and the edges are annotated with the roles that relate the connected individuals. A tableau can also be seen as an ABox, i.e., a set of (class and role) assertions. This graph is expanded by applying a set of consistency preserving expansion rules until no more rules are applicable. However, some expansion rules are non-deterministic and their application results in a set of tableaux. Therefore, the tableau algorithm manages a forest of tableau graphs and terminates when all the graphs are fully expanded.

Extensions of the standard tableau algorithm allow the computation of explanations for a query associating sets of axioms representing the set of explanations to each annotation of each node and edge. The set of annotations for a node $n$ is denoted by $\cL(n)$, analogously, the set of annotations of an edge $(n,m)$ is denoted by $\cL(n,m)$. A recent extension represents explanations by means of a Boolean formula~\cite{DBLP:journals/logcom/BaaderP10}.
In particular every node (edge) annotation, which is an assertion $a=n:C$ ($a=(n,m):R$) with $C\in\cL(n)$ ($R\in\cL((n,m))$), is associated with a label $lab(a)$ that is a monotone Boolean formula over $var(\cK)$. In the initial tableau, every assertion $a\in \cK$ is labeled with variable $var(a)$, and assertion $\neg Q$ is added with label $\top$.

The tableau is then expanded by means of expansion rules. In \cite{DBLP:journals/logcom/BaaderP10} a rule is of the form
$$(B_0,S)\rightarrow\{B_1,...,B_l\}$$
where the $B_i$s are finite sets of assertions possibly containing variables and $S$ is a finite set of axioms. 
Assertions have variables for concepts, roles and individuals, when $B_0$ can be unified with an assertion in the tableau and the set of axioms $S\in \cK$, then the rule can be applied to the tableau. Before applying the rule, all variables in assertions in $B_i$ are instantiated.

\begin{example}
\label{people+pets-tab}
In this example we show the tableau algorithm in action on an extract of the KB  of Example~\ref{people+pets}, and the query $Q = kevin:natureLover$.
$$\begin{array}{clccl}
 (1)&\exists hasAnimal.Pet \sqsubseteq NatureLover &\ & (2)&tom: Cat\\
 (3)& (kevin,tom):hasAnimal &\ & (4)&Cat\sqsubseteq Pet
 \end{array}$$
  The initial tableau, shown on the left hand side of Figure~\ref{fig:tab-expansion}, contains the nodes for $kevin$ and $tom$. The node for $tom$ is annotated with the concept $Cat$ due to axiom $(2)$, while the node for $kevin$ is annotated with the concept $\neg NatureLover$, due to the query $Q$. Moreover, the edge between the two nodes is annotated with the role $hasAnimal$, due to axiom $(3)$. The final tableau, obtained after the application of the expansion rules, is shown on the right hand side of Figure~\ref{fig:tab-expansion}. In this tableau, the node for $tom$ is also annotated with the concept $Pet$, and the node for $kevin$ with the concepts $\exists hasAnimal.Pet$ and $NatureLover$.
  \begin{figure}[hbt]
$$\xymatrix@=6mm
{*=<120pt,20pt>[F-:<3pt>]{kevin\ \ :\ \ \parbox{55pt}{$\neg NatureLover$}}
\ar@/^/@{->}[dd]_{\parbox{50pt}{\begin{flushright}hasAnimal                                         \end{flushright}}}& && *=<120pt,40pt>[F-:<3pt>]{kevin\ \ :\ \ \parbox{65pt}{\footnotesize $\exists hasAnimal.Pet$\\$NatureLover$\\$\neg NatureLover$}}\ar@/^/@{->}[dd]_{\parbox{50pt}{\begin{flushright}hasAnimal\end{flushright}}}\\ 
&\ar@{=>}[r]&&\\
 *=<60pt,20pt>[F-:<3pt>]{tom\ \ :\ \ Cat}&&&*=<65pt,30pt>[F-:<3pt>]{tom\ \ :\ \ \parbox{15pt}{$Cat$\\$Pet$}}
}
$$
\caption{Expansion of the tableau for the KB of Example~\ref{people+pets-tab}.\label{fig:tab-expansion}}
\end{figure}
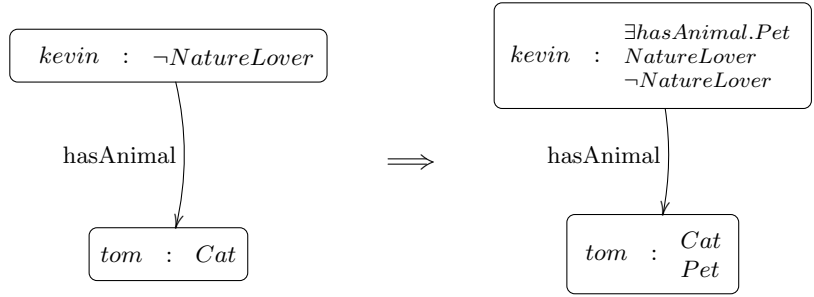
\end{example}

Rules can be divided into two sets: deterministic and non-deterministic. In the first type, $l=1$ and all the ground assertions in $B_1$ are inserted in the tableau to which the rule is applied, while in the second type $l>1$, meaning that it creates $l$ new tableaux, one for each $B_i$, and adds to the $i$-th tableau the ground assertions in $B_i$.

In order to explain the conditions that allow for the application of a rule we need first some definitions.

\begin{definition}
\label{def:psi-insert}
 Let $A$ be a set of labeled assertions and $\psi$ a monotone Boolean formula. The assertion $a$ is $\psi$-\textit{insertable} into $A$ if either $a\notin A$, or $a \in A$ but $\psi \not\models lab(a)$. 
 Given a set $B$ of assertions and a set $A$ of labeled assertions, the set of $\psi$-\textit{insertable} elements of $B$ into $A$ is defined as $ins_{\psi}(B,A):= \{b \in B | \textit{b is }\psi$-\textit{insertable}$\textit{ into A}\}$.

The result of the operation of $\psi$-\textit{insertion} of $B$ into $A$ is the set of labeled assertions $A\uplus_{\psi}B$ containing assertions in $A$ and those specified in $ins_{\psi}(B,A)$ opportunely labeled, i.e., the label of assertions in $A\setminus ins_{\psi}(B,A)$ remain unchanged, assertions in  $ins_{\psi}(B,A)\setminus A$ get label $\psi$ and the remaining $b_i$s get the label $\psi \vee lab(b_i)$.
\end{definition}

\begin{example}
 Consider the KB and the query of Example \ref{people+pets2}. After finding the first explanation for the query, which is $\{C_2,E_1,E_3,C_1\}$, the tableau contains the set of assertions 
 $A=\{\neg(kevin:NatureLover), kevin:NatureLover, \fluffy:Cat, tom:Cat\}$ 
 with labels $lab(\neg(kevin:NatureLover))=\top$, $lab(kevin:NatureLover)=C_2\wedge E_1\wedge E_3\wedge C_1$, $lab(\fluffy:Cat)=E_1$ and $lab(tom:Cat)=E_2$. Suppose we want to insert the assertion $kevin:NatureLover$ into $A$, and $\psi$ is $C_3\wedge E_2\wedge E_3\wedge C_1$. Since this formula does not imply $lab(kevin:NatureLover)$, then $kevin:NatureLover$ is $\psi$-\textit{insertable} into $A$ and its insertion changes the label $lab(kevin:NatureLover)$ to the disjunction of the two formulas, i.e., $lab(kevin:NatureLover)=((C_2\wedge E_1)\vee (C_3\wedge E_2) )\wedge E_3\wedge C_1$.
\end{example}

\noindent
We also need the concept of \emph{substitution}. A substitution is a mapping $\rho:V\rightarrow D$, where $V$ 
is a finite set of logical variables and $D$ is a countably infinite set of 
constants that contains all the individuals in the KB and all the anonymous individuals created by the application of the rules. A substitution can also be seen as a set of ordered couples in the obvious way.
Variables are seen as placeholders for individuals in the assertions. For example, an assertion can be $x:C$ or $(x,y):R$ where $C$ is a concept, $R$ is a role and $x$ and $y$ are variables.
Let $x:C$ be an assertion with variable $x$ and $\rho = \{ x\rightarrow c\}$ a substitution, then $(x:C)\rho$ denotes the assertion obtained by replacing variable $x$ with its $\rho$-\textit{image}, i.e. $(x:C)\rho=(c : C)$. A substitution $\rho'$ \emph{extends} $\rho$ if $\rho\subseteq\rho'$. A rule $(B_0,S)\rightarrow\{B_1,...,B_l\}$ can be applied to the tableau $T$ with a substitution $\rho$ on the variables occurring in $B_0$ if $S\subseteq \cK$, and $B_0\rho\subseteq A$. An applicable rule is applied by generating a set of $l$ tableaux with the $i$-th obtained by the $\psi$-\textit{insertion} of $B_{i}\rho'$ in $T$, where $\rho'$ is a substitution extending $\rho$. In the case of variables not occurring in $B_0$ (fresh variables), $\rho'$ instantiates them with new individuals which do not appear in the KB. These individuals are also called \emph{anonymous}.

\begin{example}
\label{ex:rule_application}
Consider, for example, the rule $\exists$ defined as 
$$(\{(x:\exists S.C),miss(z,\{(x,z):S,z:C\})\},\{\})\rightarrow\{\{anon(y),((x,y):S),(y:C)\}\}$$
 handling existential restrictions. Informally, ``\textit{if $(x:\exists S.C)\in A$, but there is no individual name $z$ such that $z:C$ and $(x,z):S$ in $A$, then $A=A\cup\{((x,y):S),(x:C)\}$ where $y$ is an individual name not occurring in $A$}''. If $A$ does not contain two assertions that match $((x,z):S),(z:C))$, a fresh variable $y$ is instantiated with a new fresh individual. Thus, if $A=\{x:\exists S.C, (a,b):S\}$ the rule can be applied to $A$ with substitution $\rho=\{x\rightarrow a, y\rightarrow c\}$ with $c$ a new anonymous individual. After the application of the rule $A'=A\cup\{(a,c):S, c:C\}$.
\end{example}

\noindent
However, the discussion above does not ensure that rules such as that of Example~\ref{ex:rule_application} are not applied again to $A'$ creating new fresh individuals. In fact,  just checking whether the new assertions are not contained in $A'$ does not prevent to re-apply the rule in the example to $A'$ generating $A'' = A'\cup\{(a,c'):R, c':C\}$.
This motivates the following definition for rule applicability.

\begin{definition}[Rule Applicability]
 \label{rule-app}
 Given a tableau $T$, a rule $(B_0,S)\rightarrow \{B_1,$ $...,B_l\}$ is applicable with a substitution $\rho$ on the variable occurring in $B_0$ if $S \subseteq \cK$, and $B_{0}\rho \subseteq A$, where $A$ is the set of
 assertions of the tableau, and, for every $1 \leq i \leq l$ and every substitution $\rho'$ on the variables occurring in $B_0 \cup B_i$ extending $\rho$ we have $B_{i}\rho'\nsubseteq A$.
\end{definition}

\noindent
We can now define also rule application.

\begin{definition}[Rule Application]
Given a forest of tableaux $\cF$ and a tableau $\cT\in\cF$ representing the set of assertions $A$ to which a rule is applicable with substitution $\rho$, the application of the rule leads to the new forest $\cF'=\cF\setminus\cT\cup_{i=1}^n\cT_i^\psi$. Each $\cT_i^\psi$ contains the assertions in $A\uplus_{\psi}B_{i}\rho'$, where $\rho'$ is a substitution on the variables occurring in $R$ that extends substitution $\rho$ and maps variables of $R$ to  new distinct anonymous individuals, i.e. individuals not occurring in $A$. The rule is applied for each possible $\rho$ given by $A$.
\end{definition}

\noindent
After the full expansion of the forest of tableaux, i.e., when no more rules are applicable to any tableau of the forest, the pinpointing formula is built from all the clashes in the tableaux. A clash is represented by two assertions $a$ and $\neg a$ present in the tableau.

\begin{example}
 Consider Figure~\ref{fig:tab-expansion}. In the final tableau, the node for $kevin$ is annotated with the concepts $NatureLover$ and $\neg NatureLover$. This is a clash, meaning that the query $Q=kevin:NatureLover$ is true w.r.t. the KB of the Example~\ref{people+pets-tab}.
\end{example}

The pinpointing formula is built by first conjoining, for each clash, the labels of the two clashing assertions, then by disjoining the formulas for every clash in a tableau and finally by conjoining the formulas for each tableau.

In order to ensure termination of the algorithm, blocking must be used.

\begin{definition}[Blocking]\label{def:blocking}
 Given a node $N$ of a tableau, $N$ is blocked iff either $N$ is a new node generated by a rule, it has a predecessor $N'$ which contains the same annotations of $N$ and the labels of these annotations are equal, or its parent is blocked.
\end{definition}

\begin{example}
\label{tab-block}
 Let us consider the following KB.
$$\begin{array}{cl}
 (1)&C \sqsubseteq \exists R.C\\
 (2)&a:C
 \end{array}$$
  The initial tableau, shown on the left hand side of Figure~\ref{fig:tab-blocking}, contains only the node for $a$, annotated with $C$. After the application of the $\textit{unfold}$ rule, using axiom $(1)$, and of the $\exists$ rule, explained in Example~\ref{ex:rule_application}, the resulting tableau is shown on the right hand side of Figure~\ref{fig:tab-blocking}. The tableau has a new node corresponding to an anonymous individual $an_1$, which has the same annotations of its predecessor $a$. The node for $an_1$ is blocked according to Definition~\ref{def:blocking}, because further expansion of this node would lead to the creation of an infinite chain of nodes associated to new anonymous individual, all containing the same annotations $C$ and $\exists R.C$.
  \begin{figure}[hbt]
$$\xymatrix@=6mm
{
& && *=<60pt,30pt>[F-:<3pt>]{a\ \ :\ \ \parbox{20pt}{\footnotesize $\exists R.C$\\$C$}}\ar@/^/@{->}[dd]_{\parbox{50pt}{\begin{flushright}R\end{flushright}}}\\ 
*=<60pt,20pt>[F-:<3pt>]{a\ \ :\ \ C}&\ar@{=>}[r]&&\\
 &&&*=<65pt,30pt>[F-:<3pt>]{an_1\ \ :\ \ \parbox{20pt}{\footnotesize $\exists R.C$\\$C$}}
}
$$
\caption{Expansion of the tableau for the KB of Example~\ref{tab-block}.\label{fig:tab-blocking}}
\end{figure}
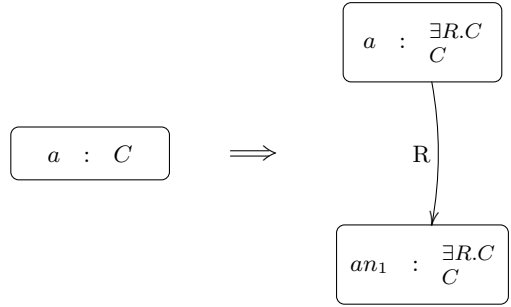
\end{example}

Then, a new definition of applicability must be given.
\begin{definition}[Rule Applicability with Blocking]
A rule is applicable if it is so in the sense of Definition~\ref{rule-app}. Moreover, if the rule adds a new node to the tableau, the node $N$ annotated with the assertion to which the rule is applied must be not blocked.
\end{definition}

\begin{theorem}[Correctness of Pinpointing Formula \cite{DBLP:journals/logcom/BaaderP10}]\label{th:corr-pin-form}
Given a KB $\cK$ and a query $Q$, for every chain of rule applications resulting in a fully expanded forest $\cF_n$, the formula built as indicated above is a pinpointing formula for the query $Q$.
\end{theorem}

 \noindent
 This approach is correct and terminating for the DL $\mathcal{SHI}$. 
 Number restrictions and nominal concepts cannot be handled by this definition of the tableau algorithm because of the definitions of rule and rule application. In fact tableau expansion rules for DLs with these constructs may merge some nodes, operation that is not allowed by the approach presented above. The authors of~\cite{DBLP:journals/logcom/BaaderP10} conjecture that the approach can be extended to deal with such constructs but, to the best of our knowledge, this conjecture has not been proved yet.
 
Until now, we have not considered transitivity axioms nor role inclusion axioms. To do so, the definition of $R$-\emph{successor} must be given.
\begin{definition}
\label{def:r-suc}
Given a role $R$, an individual $y$ is called $R$-\emph{successor} of an individual $x$ iff there is an assertion $(x,y):S$ for some sub-role $S$ of $R$.
\end{definition}
Note that, each role $R$ is a sub-role of itself. Following Definition~\ref{def:r-suc}, every assertion $(x,y):R$ indicates that $y$ is an $R$-\emph{successor} of $x$.
\begin{example}
Consider a KB containing, among the others, the following axioms:
$$\begin{array}{clclcl}
 &S\sqsubseteq R&\ &S_1\sqsubseteq S&\ &S_2\sqsubseteq S_1
 \end{array}$$
the assertion $(x,y):R$ means that $y$ is an $R$-\emph{successor} of $x$ and, therefore, that there is also the assertion $(x,y):S$, and, since $y$ is an $S$-\emph{successor} of $x$, recursively $(x,y):S_1$ and $(x,y):S_2$ as well.
\end{example}

Definition~\ref{def:r-suc} is used to deal with role inclusion when considering quantified concepts ($\exists R.C$ and $\forall R.C$) in order to correctly manage subsumption ($\exists R.C \sqsubseteq \exists S.C$ if $R\sqsubseteq S$).
The expansion rules for the tableau algorithm extended with pinpointing formula and management of $R$-\emph{successors} are shown in Figure~\ref{table:rules}. Here, $atomic(C)$ and $complex(C)$ indicate that concept $C$ is an atomic concept and a complex concept respectively. Moreover, $miss(z,Ass)$ means that there is not any individual $z$ such that the set of assertions containing $miss(z,Ass)$ does not contain the assertions defined in $Ass$. Finally, $anon(y)$ adds a new anonymous individual to the set of assertions.

 \begin{figure}[htb]
\begin{minipage}{\textwidth}
\begin{small}
\begin{tabbing}
aaa\=aaa\=aaa\=\kill
\textbf{Deterministic rules:}\\
$\textit{unfold}$: $(\{(x:C)\},\{atomic(C),(C \sqsubseteq D)\}) \rightarrow \{\{(x:D)\}\}$\\
\\
$\textit{CE}$: $(\{\},\{complex(C),(C \sqsubseteq D)\}) \rightarrow \{\{(x:(\neg C \sqcup D)| x \in \mathbf{I})\}\}$\\
\\
$\sqcap$: $(\{(x:(C_{1} \sqcap C_{2}))\},\{\}) \rightarrow \{\{(x:C_{1}),(x:C_{2})\}\}$\\
\\
$\exists$: $(\{(x:\exists S.C),miss(z,\{(x,z):S,z:C\})\},\{\})\rightarrow\{\{anon(y),((x,y):S),(y:C)\}\}$\\
\\
$\forall$: $(\{(x:\forall S.C)\},\{((x,y):S)\}) \rightarrow \{\{(y:C)\}\}$\\
\\
$\forall^{+}$: $(\{(x:\forall S.C)\},\{((x,y):R),(Trans(R)),(R\sqsubseteq S)\}) \rightarrow \{\{(y:\forall R.C)\}\}$\\
\\
\textbf{Non-deterministic rules:}\\
$\sqcup$: $(\{(x:(C_{1} \sqcup C_{2}))\},\{\}) \rightarrow \{\{(x:C_1)\},\{(x:C_2)\}\}$\\
\end{tabbing}

\end{small}
\end{minipage}
\caption{Tableau expansion rules for DL $\mathcal{SHI}$~\protect\cite{DBLP:journals/sLogica/BaaderS01}. 
For each rule, the name and formal definition are shown. 
In the rules, on the left of the arrow there are the assertions already present in the tableau, the axioms and the conditions necessary for the rule to be executed. On the right, there are the new assertions to be added in the tableau.
\label{table:rules}}
\end{figure}

As reported in~\cite{DBLP:journals/sLogica/BaaderS01}, the \textit{unfold} rule considers only subsumption axioms where the sub-class $C$ is an atomic concept. The \textit{CE} rule is used in the case that the sub-class $C$ is not atomic, in such a case the \textit{unfold} rule might lead to an exponential blow-up. The \textit{CE} rule applies every subsumption axiom where the sub-class is complex to every individual of the KB.

While the $\sqcup$ and $\sqcap$ rules are easily understandable, the $\exists$ rule ensures that there exists at least one individual connected to $x$ by role $R$ belonging to class $C$. The $\forall$ rule ensures that every individual connected to $x$ by role $R$ belongs to the concept $C$ specified by the assertion, while the $\forall^{+}$ ensures that the effects of universal restrictions are propagated as necessary in the presence of non-simple roles. It basically adds $y:\forall R.C$ iff $y$ is an $R$-\emph{successor} of $x$ such that $x:\forall S.C$ is included in the set of initial assertions and $R$ is a transitive sub-role of~$S$.

We refer to  \cite{DBLP:journals/sLogica/BaaderS01} for a detailed discussion on the tableau algorithm for DLs and its rules.

\section{\trillpbdd}
\label{sec:systems}

%

As TRILL and \trillp, \trillpbdd implements the tableau algorithm described in the previous section. In particular, \trillpbdd\ shares the same basis of \trillp\ because they both build the pinpointing formula representing the answer to queries. 
Differently from \trillp, \trillpbdd labels the assertions with a BDD representing the pinpointing formula instead of  the formula itself.
$\psi$-\textit{insertabili\-ty} can be checked in this case without resorting to a SAT solver.
In fact, suppose the tableau contains assertion $A$ labeled with BDD $B$, and we want to add BDD $B'$ to the label of assertion $A$, where $B'$ represents the formula $\psi$.
If $A$ is $\psi$-\textit{insertable}, the result is that assertion $A$ in the tableau will have the BDD obtained by disjoining $B$ and $B'$, $B\vee B'$, as label. $A$ is $\psi$-\textit{insertable} if $B'\not\models B$. We have that $B' \models B \Leftrightarrow B\vee B'\equiv B$.
Since BDDs are a canonical representation of Boolean formulas, $B\vee B'\equiv B$ iff $B\vee B'= B$, so we can avoid the SAT test by computing the disjunction of BDDs and checking whether the result is the same as the first argument, i.e., the two BDDs represent the same Boolean formula or, in other words, they represent two Boolean formulas which have the same truth value. If this is not the case, we can insert the formula in the tableau with BDD $B\vee B'$ which is already computed. 

\begin{theorem}[TORNADO's Correctness]
 Given a KB $\cK$ and a query $Q$, the probability value returned by TORNADO when answering query $Q$ corresponds to the probability value for the query $Q$ computed accordingly to the DISPONTE semantics.
\end{theorem}
\begin{proof}
The proof of this theorem follows from Theorem~\ref{th:corr-pin-form}. Since the pinpointing formula of query $Q$ w.r.t. the KB $\cK$ corresponds to the set $Expls(\cK,Q)$ of explanations, also their translation into BDDs is equivalent. TORNADO implements the tableau algorithm computing the pinpointing formula and represents such formula directly with BDDs built during  inference,  hence the probability computed from $B$ is correct w.r.t. the semantics.
\end{proof}

\subsection{Implementation of \trillpbdd}
First, we describe the common parts of \trillpbdd, \trillp\ and TRILL and then we show the differences.
The code of all three systems is available at \url{https://github.com/rzese/trill} and can be tested online with the TRILL on SWISH web application at \url{http://trill.ml.unife.it/}.
Figure~\ref{fig:trill-on-swish} shows the TRILL on SWISH interface.
\begin{figure}
 \includegraphics[width=1\textwidth]{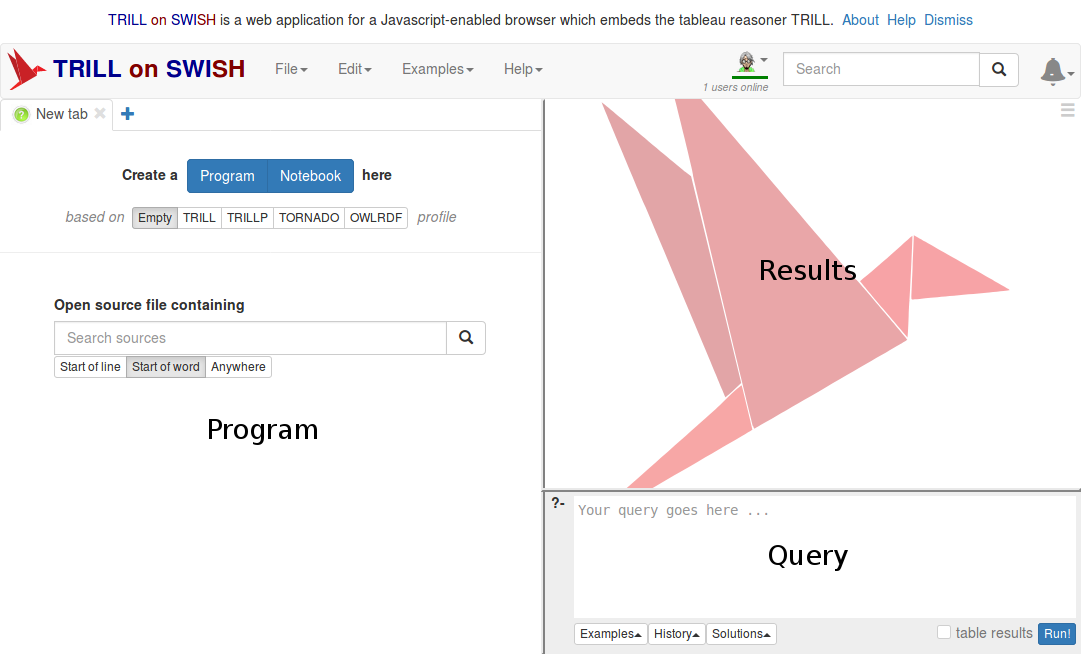}
 \caption{TRILL on SWISH interface.\label{fig:trill-on-swish}}
\end{figure}

All systems allow the use of two different syntaxes for axioms: OWL/RDF and Prolog.
The first can be used by exploiting the predicate \verb|owl_rdf/1|, whose argument is a string containing the KB in OWL/RDF.
The Prolog syntax is borrowed fro the Thea\footnote{\url{http://vangelisv.github.io/thea/}} library,  similar to the Functional-Style Syntax of OWL~\cite{motik2009owl} and represents axioms as Prolog atoms. For example, the axiom
$$Cat\sqsubseteq Pet$$
stating that \textit{cat} is subclass of \textit{pet} can be expressed as
\begin{verbatim}
subClassOf(cat,pet)
\end{verbatim}
while the axiom
$$Pet \equiv (Animal \sqcup \neg Wild)$$
stating that \textit{pet} is equivalent to the intersection of classes \textit{animal} and not \textit{wild} can be expressed as:
\begin{verbatim}
equivalentClasses([pet,
           intersectionOf([animal,complementOf(wild)])])
\end{verbatim}

\noindent In order to represent the tableau, the systems use a pair $Tableau=(A, T)$, where $A$ is a list containing assertions labeled with the 
corresponding pinpointing formula and 
$T$ is a triple ($G$, $RBN$, $RBR$) in which $G$ is a directed graph that encodes the structure of the tableau, 
$RBN$ is a red-black tree (a key-value dictionary), where a key is a pair of individuals and its value is the set of roles that connect the two individuals, and 
$RBR$ is a red-black tree, where a key is a role and its value is the set of pairs of individuals that are linked by the role.
These structures are built and handled by using two Prolog built-in libraries, one tailored for unweighted graphs, used for the structure of the tableau $G$, and one for red-black trees, used for the two dictionaries $RBN$ and $RBR$.
From the data structure $T$ we can quickly find the information needed during the execution of the tableau algorithm and check blocking conditions through predicates  
\texttt{nominal/2} and \texttt{blocked/2}. 
These predicates take as input a nominal individual $\mathit{Ind}$ and a tableau $(A, T)$.
For each  individual $\mathit{Ind}$ in the ABox, the  atom $nominal(\mathit{Ind})$ is added to $A$ in the initial tableau in order to rapidly check whether a node is associated with an anonymous individual or not.

All non-deterministic rules are implemented using a predicate of the form\linebreak $rule\_name(Tab0, TabList)$, that takes as input the current tableau $Tab0$ and returns the list of tableaux $TabList$ created by the application of the rule to $Tab0$.
Deterministic rules are 
implemented by a predicate $rule\_name(Tab0,Tab)$ that returns a single tableau
$Tab$ after the application of $rule$ to $Tab0$.

Since the order of rule application does not influence the final result, deterministic rules are applied first and then the non-deterministic ones in order to delay as much as possible the generation of new tableaux. Among deterministic rules, $\forall$, $\forall^{+}$, and $\exists$ are applied as last rules~\cite{DBLP:journals/sLogica/BaaderS01}.
After the application of a deterministic rule, a cut avoids backtracking to other possible choices for the deterministic rules. 
Then, non-deterministic rules are tried sequentially. After the application of a non-deterministic rule, a cut is performed to avoid backtracking to other rule choices and a tableau from the list is 
non-deterministically chosen with \texttt{member/2}. 
If no rule is applicable, rule application stops and returns the current tableau, otherwise a new round of rule application is performed.

The labels of assertions are combined in \trillp using functors \verb|*/1| and \verb|+/1| representing conjunction and disjunction respectively. Their argument is the list of operands. For example the formula of Example~\ref{people+pets4} can be represented as
\begin{center}
\texttt{+([$C_1$,$E_3$,+([*([$C_2$,$E_1$]),*([$C_3$,$E_2$])])])} 
\end{center}

\noindent
$\psi$-\textit{insertability} is checked in \trillp by conjoining the formula we want to add with the negation of the formula labeling the assertion in the tableau. If the resulting formula is satisfiable, then the assertion is $\psi$-\textit{insertable}. Predicate \texttt{test/2} checks $\psi$-\textit{insertability}: it takes as input the two formulas and calls a satisfiability library after having transformed the formulas into a suitable format.


The Boolean pinpointing formula returned by \trillp is then translated into a BDD from which the probability can be computed. 

As already seen, \trillpbdd\ avoids the steps just described by directly building BDDs. $\psi$-\textit{insertability} is checked by disjoining the current label of assertion and the new BDD found and checking whether the resulting BDD is different from the original label of the assertion. Finally, when \trillpbdd\ ends the computation of the query, the corresponding BDD is already built and can be used to calculate the probability of the query.


 BDDs are managed in Prolog by using a library developed for the system PITA~\cite{RigSwi10-ICLP10-IC,RigSwi11-ICLP11-IJ}, which interfaces Prolog to the CUDD library\footnote{\url{http://vlsi.colorado.edu/~fabio/CUDD/}} for manipulating BDDs. The PITA library offers predicates for performing Boolean operations between BDDs. 
Note that BDDs are represented in  Prolog with pointers to their root node and checking equality between BDDs can be performed by checking equality between two pointers which is constant in time.
Thus, the \texttt{test/2} predicate has only to update the BDD and check if the new BDD is different from the original one. This test is necessary to avoid entering in an infinite loop where the same assertion is inserted infinitely many times.
The code of \trillpbdd's \texttt{test/2} predicate is shown below. 
\begin{verbatim}
test(BDD1,BDD2,F) :-   % BDD1 is the new BDD,
                       % BDD2 is the BDD already in the tableau
    or_f(BDD1,BDD2,F), % combines BDD1 and BDD2 to create BDD F
    BDD2 \== F.        % checks if F is different from BDD2
\end{verbatim}

The time taken by Boolean operations between BDDs and the size of the results depend on the ordering of the variables. A smart order can significantly reduce the time and size of the results. However, the problem of finding the optimal order is coNP-complete \cite{Bryant:1986:GAB:6432.6433}. For this reason, heuristic methods are used to choose the ordering. CUDD for example offers symmetry detection or genetic algorithms. Reordering can be executed when the user requests it or automatically by the package when the number of nodes reaches a certain threshold. The threshold is initialized and automatically tuned after each reordering. We refer to the documentation\footnote{\url{http://www.cs.uleth.ca/~rice/cudd_docs/}} of the library for detailed information about each implemented heuristic.

 It is important to note that CUDD groups BDDs in environments called BDD managers. We use a single BDD manager for each query. When a reordering is made, all the BDDs of the BDD manager are reordered. So the difference test can compare the two pointers.

For \trillpbdd\ we chose the group sifting heuristic \cite{DBLP:conf/iccad/PandaS95} for the order selection, natively available in the CUDD package. 


TORNADO never forces the reordering and uses CUDD automatic dynamic reordering. 
However, as one can see from the experimental results presented in the next section, \trillpbdd\ is able to achieve good results using the default settings.

\section{Experiments}
\label{sec:exp}

We performed two experiments, the first one regarding non-probabilistic inference, the second one regarding probabilistic inference. 

In the first experiment we compared TRILL, \trillp, \trillpbdd, and BORN with the non-probabilistic reasoners Pellet~\cite{DBLP:journals/ws/SirinPGKK07}, Konclude\footnote{\url{http://derivo.de/produkte/konclude/}}~\cite{StLG14a}, HermiT~\cite{shearer2008hermit}, Fact++~\cite{tsarkov2006fact++}, and JFact\footnote{\url{http://jfact.sourceforge.net/}}.
Konclude can check the consistency of a KB and satisfiability of concepts, define the class hierarchy of the KB and find all the classes to which a given individual belongs. However, it cannot directly answer general queries or return explanations. On the other hand, Pellet, HermiT, Fact++, and JFact answer general queries and can be used for returning all  explanations. 
To find explanations, once the first one is found, Pellet, HermiT, Fact++, and JFact use the Hitting Set Tree (HST) algorithm~\cite{DBLP:journals/ai/Reiter87} to compute the others by repeatedly removing axioms one at a time and invoking the reasoner. This algorithm is implemented in the OWL Explanation library\footnote{\url{https://github.com/matthewhorridge/owlexplanation}}~\cite{Horridge09theowl}. 

 Basically, it takes as input one explanation, randomly chooses one axiom from the explanation and removes it form the KB. At this point the HST algorithm calls the reasoner to try to find a new explanation w.r.t. the reduced KB. If a new explanation is found, a new axiom from this new explanation is selected and removed from the reduced KB, trying to find a new explanation. Otherwise, the removed axiom is added to the reduced KB and a new axiom is selected to be removed from the last explanation found. The HST algorithm stops when all the axioms form all the explanations have been tested. 
Therefore, to find a new explanation at every iteration, OWL Explanation, for HermiT, Fact++ and JFact, uses a black box approach, i.e. a reasoner-independent approach. Whereas Pellet uses a built-in approach to find them, which is, however, the HST algorithm implemented in OWL Explanation slightly modified. On the other hand, Konclude does not implements the OWL API interface that we used for the implementation of the black box algorithm. Moreover, since it does not return explanations, the black box approach described above cannot be directly applied. In order to use Konclude for finding all possible explanations would require significant development work with a careful tuning of the implementation, which is outside of the scope of this paper. Therefore, we decided to include Konclude only in tests where we are not interested in finding all the explanations.

In the second experiment we compared TRILL, \trillp, \trillpbdd, BORN, BUNDLE and PRONTO. 
While TRILL, \trillp, BUNDLE and \trillpbdd\ all follow the DISPONTE semantics, PRONTO and BORN are based on different semantics.
PRONTO uses  P-$\mathcal{SHIQ}(\mathbf{D})$ \cite{DBLP:journals/ai/Lukasiewicz08}, a language based on Nilsson's probabilistic logic \cite{DBLP:journals/ai/Nilsson86}, that defines probabilistic interpretations instead of a single probability distribution over theories (such as DISPONTE).
BORN uses $\mathcal{BEL}$, that extends the $\mathcal{EL}$ Description Logic with Bayesian networks and is strongly related to DISPONTE. In fact, DISPONTE is a special case of $\mathcal{BEL}$ where (1) every axiom corresponds to a single Boolean random variable, while $\mathcal{BEL}$ allows a set of Boolean random variables; and (2) the Bayesian network has no edges, i.e., all the variables are independent.
This special case greatly simplifies reasoning while still achieving significant expressiveness.
Note that if we need the added expressiveness of $\mathcal{BEL}$, as shown in \cite{ZesBelRig16-AMAI-IJ}, the Bayesian network can be translated into an equivalent one where all the random variables are mutually unconditionally independent, so that the KB can be represented with DISPONTE.

Because of the above differences, the comparison with PRONTO and BORN is only meant to provide an empirical comparison of the difficulty of reasoning under the various semantics. 


\trillp is implemented  both in YAP and SWI-Prolog, while \trillpbdd only in SWI-Prolog, thus all tests were run with the SWI-Prolog version of the \trillp\footnote{The SWI-Prolog version exploits the solver contained in the  \texttt{clpb} (\url{http://www.swi-prolog.org/pldoc/man?section=clpb}) library.}. 
Pellet, BUNDLE and BORN are implemented in Java. BORN needs ProbLog to perform inference, we used version 2.1. To get the fairest  results, the measured running time does not include the start-up time of the Prolog interpreter and of the Java virtual machine, but only inference and KBs loading.

All tests were performed on the HPC System Marconi\footnote{\url{http://www.hpc.cineca.it/hardware/marconi}} equipped with Intel Xeon E5-2697 v4 (Broadwell) @ 2.30 GHz, using 8 cores for each test.

\subsection{Non-Probabilistic Inference}
We performed three different tests for the non-probabilistic case. One with KBs modeling real world domains and two with artificial KBs.

\paragraph{Test 1}
\label{test:np-1}
We used four real-world KBs as in \cite{ZesBelRig16-AMAI-IJ}:
\begin{itemize}
 \item BRCA\footnote{\url{http://www2.cs.man.ac.uk/~klinovp/pronto/brc/cancer_cc.owl}}, which models the risk factors of breast cancer depending on many factors such as age and drugs taken;
 \item an extract of  DBPedia ontology\footnote{\url{http://dbpedia.org/}}, containing structured information of Wikipedia, usually those contained in the information box on the right hand side of a page;
 \item BioPAX level 3\footnote{\url{http://www.biopax.org/}}, which models metabolic pathways;
 \item Vicodi\footnote{\url{http://www.vicodi.org/}}, which contains information on European history and models historical events and important personalities.
\end{itemize}
We used a version of the DBPedia and BioPAX KBs without the ABox and a version of  BRCA and Vicodi with an ABox containing 1 individual and 19 individuals respectively.
We randomly created 50 subclass-of queries for DBPedia and BioPAX and 50 instance-of queries for the other two, ensuring each query had at least one explanation.
We ran each query with two different settings.

In the first setting, we used the reasoners to answer Boolean queries. We compared Konclude, Pellet, HermiT, Fact++, JFact, BORN, TRILL, \trillp and \trillpbdd. In this setting, Konclude has an advantage because it is optimized to test concept satisfiability. TRILL provides a predicate for answering yes/no to queries by checking for the existence of an explanation. On the other hand, \trillp and \trillpbdd  are used by checking whether the output formula  is satisfiable and BORN by checking that the probability of the query is not $0$. 

For all the considered reasoners except Konclude, we used the queries generated as described above. For Konclude, in order to perform tests as close as possible with the other competitors, for each subclass-of test with query $C\sqsubseteq D$ we extended the KB with one test concept defined as $\neg(C\wedge\neg D)$, while for each instance-of query $a:C$, where $a$ belongs to the concepts $C_1,...,C_n$, we extended the KB with one test concept defined as $\neg(C_C\wedge\neg D)$, where $C_C$ is defined as the intersection of $C_1,...,C_n$.

Table \ref{table:res-1-np-1-1} shows the average running time and its standard deviation in seconds to answer queries on each KB.
On BRCA, \trillp performs worse than TRILL and \trillpbdd since the SAT solver is repeatedly called  with complex formulas. 
Konclude is the best on all KBs except DBPedia, where \trillpbdd performs similarly. \trillpbdd is the second faster on BioPAX and BRCA, while \trillp is the second fastest algorithm on Vicodi. TRILL, \trillp, \trillpbdd and Konclude outperform Pellet, HermiT, Fact++ and JFact. 

\begin{table}
	\caption{Average time (in seconds) for answering Boolean queries  with the reasoners Pellet, HermiT, Fact++ and JFact, Konclude, BORN, TRILL, \trillp and \trillpbdd\ in Test 1, w.r.t. 4 different KBs. Each cell contains the running time $\pm$ its standard deviation. ``n.a.'' means not applicable. Bold values highlight the fastest reasoner for each KB.}
	\label{table:res-1-np-1-1}  
	
	\begin{center}
		\begin{tabular}{l|cccc}
			\hline\hline
			& BioPAX & DBPedia & Vicodi  & BRCA\\\hline
			Pellet& 1.502 $\pm$ 0.082 & 0.965 $\pm$ 0.083 & 1.334 $\pm$ 0.072 & 2.148 $\pm$ 0.12 \\
			BORN & n.a. & 6.142 $\pm$ 0.057 & n.a. & n.a. \\
			Konclude & \textbf{0.025 $\pm$ 0.004} & 0.013 $\pm$ 0.002 & \textbf{0.012 $\pm$ 0.001 } & \textbf{0.018 $\pm$ 0.001}\\
			Fact++ & 1.405 $\pm$ 0.126 & 1.230 $\pm$ 0.085 &  1.276 $\pm$ 0.131 & 1.465 $\pm$ 0.094 \\
			HermiT& 6.572 $\pm$ 0.0367 & 3.917 $\pm$ 0.279 &  6.313 $\pm$ 0.557 & 8.622 $\pm$ 0.603\\
			JFact& 1.895 $\pm$ 0.058 & 1.625 $\pm$ 0.9 &  1.772 $\pm$ 0.087 & 2.832 $\pm$ 0.1\\
			TRILL & 0.108 $\pm$ 0.047 & 0.106 $\pm$ 0.012 &  0.044 $\pm$ 0.018 & 0.800 $\pm$ 0.021\\
			\trillp & 0.109 $\pm$ 0.03 & 0.139 $\pm$ 0.007 &  0.038 $\pm$ 0.020 & 1.486 $\pm$ 0.039\\
			\trillpbdd & 0.105 $\pm$ 0.055 & \textbf{0.012 $\pm$ 0.005} &  0.041 $\pm$ 0.021 & 0.082 $\pm$ 0.018\\
			\hline\hline
		\end{tabular}  
	\end{center}
	
\end{table}

%
%

In the second setting, we collected all the explanations, that is the fairest comparison since both \trillp and \trillpbdd explore all the search space during inference, and so does BORN. 
We ran Pellet, HermiT, Fact++, JFact, BORN, TRILL, \trillp and \trillpbdd, while Konclude was not considered because we are interested here in finding all the explanations.
Table \ref{table:res-1-np-2-1} shows, for each ontology, the average number of explanations, 
and the average time in seconds to answer the queries for all the considered reasoners, together with the standard deviation. The values for BORN are taken from Table~\ref{table:res-1-np-1-1} because the check on the final probability for BORN can be neglected.

BRCA and DBPedia get the highest average number of explanations as they contain mainly subclass axioms between complex concepts. 

In general TRILL, \trillp and \trillpbdd perform similarly to the first setting, while Pellet, HermiT, Fact++ and JFact are slower than in the first setting. BORN could be applied only to DBPedia given that it can only handle  $\mathcal{EL}$ DLs.
On BRCA, \trillp performs worse than TRILL and \trillpbdd since the SAT solver is repeatedly called  with complex formulas. 
\trillpbdd is the best on all KBs except Vicodi, thanks to the compact encoding of explanations via BDDs and the non-use of a SAT solver, while \trillp is the fastest algorithm on Vicodi and the second fastest algorithm on BioPAX. In all the other cases, TRILL achieves the second best results.

While TRILL, \trillp, \trillpbdd terminate within one second (except for \trillp on BRCA), the remaining reasoners
are slower. This is probably due to the approach used to find explanations (OWL Explanation library for HermiT, Fact++ and JFact, and a built-in approach for Pellet): the use of satisfiability reasoner in the HST may be less efficient than a reasoner specifically designed to return explanations.
\begin{table}[htb]
	\caption{Average number of  explanations and average time (in seconds) for computing all the explanations of queries with the reasoners Pellet, HermiT, Fact++ and JFact, BORN, TRILL, \trillp and \trillpbdd\ in \nameref{test:np-1}, w.r.t. 4 different KBs. Each cell contains the running time $\pm$ its standard deviation. ``n.a.'' means not applicable. Bold values highlight the fastest reasoner for each KB.}
	\label{table:res-1-np-2-1}  
	\begin{center}
		\begin{tabular}{l|cccc}
			\hline\hline
			& BioPAX & DBPedia & Vicodi  & BRCA\\
			Avg. N. Expl.& 3.92 & 16.32 & 1.02 & 6.49 \\\hline
			Pellet & 1.954 $\pm$ 0.363 & 1.624 $\pm$ 0.637 & 1.734 $\pm$ 0.831 & 7.038 $\pm$ 2.952\\
			BORN & n.a. & 6.142 $\pm$ 0.238 & n.a. & n.a.\\
			Fact++ & 3.837 $\pm$ 1.97 & 5.000 $\pm$ 1.266 &  2.803 $\pm$ 1.13 & 8.218 $\pm$ 3.754\\
			HermiT & 11.798 $\pm$ 4.069 & 18.879 $\pm$ 16.754 & 9.331 $\pm$ 9.509 & 25.034 $\pm$ 10.855\\
			JFact & 5.395 $\pm$ 3.913 & 12.274 $\pm$ 4.99 & 4.771 $\pm$ 4.03 & 18.068 $\pm$ 27.280\\
			TRILL & 0.137 $\pm$ 0.042 & 0.108 $\pm$ 0.01 &  0.049 $\pm$ 0.026 & 0.805 $\pm$ 0.024\\
			\trillp & 0.110 $\pm$ 0.043 & 0.139 $\pm$ 0.006 &  \textbf{0.039 $\pm$ 0.018} & 1.507 $\pm$ 0.045\\
			\trillpbdd & \textbf{0.106 $\pm$ 0.039} & \textbf{0.012 $\pm$ 0.008} & 0.041 $\pm$ 0.021 & \textbf{0.083 $\pm$ 0.031}\\
			\hline\hline
		\end{tabular}
	\end{center}
\end{table}


\paragraph{Test 2}
\label{test:np-2}
Here we followed the idea presented in Section 3.6 of \cite{ZesBelRig16-AMAI-IJ}, for investigating the effect of the non-determinism in the choice of rules. In particular, we artificially created a set of KBs of increasing size of the following form:\\
\begin{align}
&C_{1,1} \sqsubseteq C_{1,2} \sqsubseteq ... \sqsubseteq C_{1,n}\sqsubseteq C_{n+1}\notag\\
&C_{1,1} \sqsubseteq C_{2,2} \sqsubseteq ... \sqsubseteq C_{2,n}\sqsubseteq C_{n+1}\notag\\
&C_{1,1} \sqsubseteq C_{3,2} \sqsubseteq ... \sqsubseteq C_{3,n}\sqsubseteq C_{n+1}\notag\\
&...\notag\\
&C_{1,1} \sqsubseteq C_{m,2} \sqsubseteq ... \sqsubseteq C_{m,n}\sqsubseteq C_{n+1}\notag
\end{align}
with $m$ and $n$ varying in 1 to 7. The assertion $a:C_{1,1}$ is then added and the queries $Q=a:C_{n+1}$ are asked. For each KB, $m$ explanations can be found and every explanation contains $n+1$ axioms, $n$ subclass-of axioms and 1 assertion axiom. The idea is to create an increasing number of backtracking points in order to test how Prolog can improve the performance when collecting all explanations with respect to procedural languages.
For this reason, Konclude  was not considered in this test.

Table \ref{table:res-2-np} reports the average running time on 100 query executions for each system and KB when computing all the explanations for the query $Q$. Columns correspond to $n$ while rows correspond to $m$. As in \cite{ZesBelRig16-AMAI-IJ}, we set a time limit of 10 minutes for query execution.  In these cases, the corresponding cells are filled in with ``--''.

Results show that even small KBs may cause large running times for Pellet, HermiT, Fact++, and JFact , while BORN, TRILL, \trillp and \trillpbdd\  scale much better.

For $m=1, 2$, TRILL, \trillp and \trillpbdd take about the same time; for $m>2$,  \trillp's becomes slower due to the use of the SAT solver. 

BORN takes about 3.5 seconds in all cases, which is  probably due to ProbLog exploiting Prolog backtracking as well.

\begin{table}
  \caption{Average time (in seconds) for computing  all the explanations with the reasoners Pellet, BORN, Fact++, JFact, HermiT, TRILL, \trillp and \trillpbdd\ in \nameref{test:np-2}.  ``--'' means that the execution timed out (10 minutes). Columns correspond to $n$ while  rows correspond to $m$. In bold the best time for each size.}
    \label{table:res-2-np}  
{

\begin{footnotesize}
\begin{center}
\begin{tabular}{llccccccc}
\hline\hline
& \textbf{Reasoner} & \textbf{1} & \textbf{2} & \textbf{3} & \textbf{4} & \textbf{5} & \textbf{6} & \textbf{7}\\
\noalign{\vspace {.15cm}}
\multirow{8}{*}{\textbf{1}} & \textbf{Pellet} & 0.277 & 0.291 & 0.289 & 0.284 & 0.288 & 0.291 & 0.295\\
 & \textbf{BORN} & 3.622 & 3.547 & 3.566 & 3.658 & 3.581 & 3.585 & 3.586\\
 & \textbf{Fact++} & 0.289 & 0.314 & 0.338 & 0.366 & 0.39 & 0.412 & 0.432\\
 & \textbf{HermiT} & 0.675 & 0.787 & 0.875 & 0.961 & 1.057 & 1.157 & 1.222\\
 & \textbf{JFact} & 0.406 & 0.438 & 0.458 & 0.486 & 0.513 & 0.538 & 0.561\\
 & \textbf{TRILL} & 0.0004 & \textbf{0.0004} & 0.0005 & 0.0005 & 0.0006 & 0.0006 & 0.0007\\
 & \textbf{\trillp} & \textbf{0.0003} & \textbf{0.0004} & \textbf{0.0004} & \textbf{0.0004} & \textbf{0.0005} & \textbf{0.0005} & \textbf{0.0005}\\
 & \textbf{\trillpbdd} & 0.0006 & 0.0006 & 0.0007 & 0.0008 & 0.0008 & 0.0008 & 0.0009\\
\noalign{\vspace {.15cm}}
\multirow{8}{*}{\textbf{2}} & \textbf{Pellet} & 0.285 & 0.295 & 0.307 & 0.377 & 0.417 & 0.351 & 0.363\\
 & \textbf{BORN} & 3.565 & 3.611 & 3.59 & 3.605 & 3.602 & 3.582 & 3.618\\
 & \textbf{Fact++} & 0.381 & 0.504 & 0.621 & 0.768 & 0.938 & 1.12 & 1.328\\
 & \textbf{HermiT} & 1.005 & 1.326 & 1.573 & 1.892 & 2.257 & 2.776 & 3.232\\
 & \textbf{JFact} & 0.502 & 0.637 & 0.761 & 0.915 & 1.082 & 1.276 & 1.469\\
 & \textbf{TRILL} & 0.0005 & 0.0006 & 0.0007 & 0.0008 & 0.001 & 0.0011 & 0.0013\\
 & \textbf{\trillp} & \textbf{0.0003} & \textbf{0.0004} & \textbf{0.0004} & \textbf{0.0004} & \textbf{0.0005} & \textbf{0.0005} & \textbf{0.0005}\\
 & \textbf{\trillpbdd} & 0.0007 & 0.0008 & 0.0009 & 0.001 & 0.0011 & 0.0012 & 0.0013\\
\noalign{\vspace {.15cm}}
\multirow{8}{*}{\textbf{3}} & \textbf{Pellet} & 0.298 & 0.325 & 0.361 & 0.409 & 0.465 & 0.526 & 0.595\\
 & \textbf{BORN} & 3.551 & 3.693 & 3.626 & 3.61 & 3.601 & 3.623 & 3.647\\
 & \textbf{Fact++} & 0.532 & 0.85 & 1.336 & 2.072 & 3.15 & 4.565 & 6.38\\
 & \textbf{HermiT} & 1.384 & 2.119 & 3.295 & 4.715 & 7.095 & 10.062 & 14.111\\
 & \textbf{JFact} & 0.658 & 0.984 & 1.48 & 2.216 & 3.318 & 4.735 & 6.554\\
 & \textbf{TRILL} & \textbf{0.0006} & \textbf{0.0008} & \textbf{0.001} & \textbf{0.0012} & 0.0015 & 0.0018 & 0.0022\\
 & \textbf{\trillp} & 0.0015 & 0.0019 & 0.0023 & 0.0026 & 0.0031 & 0.0037 & 0.0043\\
 & \textbf{\trillpbdd} & 0.0007 & 0.0009 & 0.0011 & \textbf{0.0012} & \textbf{0.0014} & \textbf{0.0016} & \textbf{0.0019}\\
\noalign{\vspace {.15cm}}
\multirow{8}{*}{\textbf{4}} & \textbf{Pellet} & 0.314 & 0.381 & 0.487 & 0.647 & 0.914 & 1.476 & 2.423\\
 & \textbf{BORN} & 3.525 & 3.599 & 3.616 & 3.621 & 3.629 & 3.612 & 3.641\\
 & \textbf{Fact++} & 0.707 & 1.582 & 3.62 & 7.523 & 14.539 & 26.024 & 43.627\\
 & \textbf{HermiT} & 1.799 & 3.807 & 7.968 & 16.544 & 32.073 & 56.955 & 95.111\\
 & \textbf{JFact} & 0.832 & 1.707 & 3.77 & 7.686 & 14.526 & 25.939 & 43.168\\
 & \textbf{TRILL} & \textbf{0.0008} & \textbf{0.001} & \textbf{0.0013} & 0.0017 & 0.0022 & 0.0027 & 0.0032\\
 & \textbf{\trillp} & 0.0048 & 0.0067 & 0.0088 & 0.0116 & 0.015 & 0.0187 & 0.0235\\
 & \textbf{\trillpbdd} & 0.0009 & 0.0011 & \textbf{0.0013} & \textbf{0.0016} & \textbf{0.0019} & \textbf{0.0022} & \textbf{0.0026}\\
\noalign{\vspace {.15cm}}
\multirow{8}{*}{\textbf{5}} & \textbf{Pellet} & 0.34 & 0.488 & 0.824 & 2.054 & 5.287 & 16.238 & 45.527\\
 & \textbf{BORN} & 3.348 & 3.376 & 3.369 & 3.39 & 3.404 & 3.414 & 3.438\\
 & \textbf{Fact++} & 0.987 & 3.482 & 11.691 & 33.454 & 82.118 & 181.965 & 378.121\\
 & \textbf{HermiT} & 2.548 & 7.741 & 25.869 & 73.409 & 178.472 & 384.51 & --\\
 & \textbf{JFact} & 1.112 & 3.649 & 11.782 & 33.43 & 81.333 & 178.707 & 367.852\\
 & \textbf{TRILL} & \textbf{0.0009} & \textbf{0.0013} & 0.0018 & 0.0023 & 0.003 & 0.0037 & 0.0046\\
 & \textbf{\trillp} & 0.0077 & 0.011 & 0.0149 & 0.0202 & 0.0268 & 0.0344 & 0.0412\\
 & \textbf{\trillpbdd} & 0.001 & \textbf{0.0013} & \textbf{0.0016} & \textbf{0.002} & \textbf{0.0025} & \textbf{0.003} & \textbf{0.0035}\\
\noalign{\vspace {.15cm}}
\multirow{8}{*}{\textbf{6}} & \textbf{Pellet} & 0.379 & 0.722 & 2.876 & 17.113 & 113.869 & -- & --\\
 & \textbf{BORN} & 3.34 & 3.36 & 3.35 & 3.397 & 3.386 & 3.398 & 3.41\\
 & \textbf{Fact++} & 1.481 & 8.898 & 48.352 & 192.633 & -- & -- & --\\
 & \textbf{HermiT} & 3.679 & 19.196 & 96.009 & 365.037 & -- & -- & --\\
 & \textbf{JFact} & 1.581 & 8.779 & 43.811 & 168.683 & 591.641 & -- & --\\
 & \textbf{TRILL} & \textbf{0.0011} & 0.0016 & 0.0023 & 0.0031 & 0.004 & 0.005 & 0.0062\\
 & \textbf{\trillp} & 0.0114 & 0.0171 & 0.0241 & 0.033 & 0.0447 & 0.0553 & 0.0669\\
 & \textbf{\trillpbdd} & \textbf{0.0011} & \textbf{0.0015} & \textbf{0.002} & \textbf{0.0025} & \textbf{0.0031} & \textbf{0.0038} & \textbf{0.0046}\\
\noalign{\vspace {.15cm}}
\multirow{8}{*}{\textbf{7}} & \textbf{Pellet} & 0.454 & 1.75 & 22.053 & 582.755 & -- & -- & --\\
 & \textbf{BORN} & 3.317 & 3.355 & 3.376 & 3.366 & 3.39 & 3.412 & 3.408\\
 & \textbf{Fact++} & 2.231 & 24.131 & 183.652 & -- & -- & -- & --\\
 & \textbf{HermiT} & 5.902 & 56.758 & 406.908 & -- & -- & -- & --\\
 & \textbf{JFact} & 2.393 & 24.15 & 180.689 & -- & -- & -- & --\\
 & \textbf{TRILL} & 0.0013 & 0.002 & 0.0029 & 0.0039 & 0.0051 & 0.0065 & 0.0081\\
 & \textbf{\trillp} & 0.0164 & 0.0254 & 0.037 & 0.0518 & 0.0672 & 0.0838 & 0.1038\\
 & \textbf{\trillpbdd} & \textbf{0.0012} & \textbf{0.0018} & \textbf{0.0024} & \textbf{0.0031} & \textbf{0.0039} & \textbf{0.0049} & \textbf{0.0059}\\
 \hline\hline
\end{tabular}
\end{center}
\end{footnotesize}

}
\end{table}

\paragraph{Test 3}
\label{test:np-3}
In the third experiment we used the KB of Example~\ref{exp-expl}.
We increased $n$ from 2 to 10 in steps of 2 and we collected the running time, averaged over 50 executions.
Table \ref{table:res-3-np} shows, for each $n$, the average time in seconds taken by the systems  for computing the set of all the explanations for  query $Q$. As for \nameref{test:np-2}, we did not consider Konclude. We set a timeout of 10 minutes for each query execution, so the cells with ``--'' indicate that the timeout occurred.

Results show that \trillpbdd and BORN avoid the exponential blow-up of the other systems, and that the former achieves the best performance.
\begin{table*}[htb]
  \caption{Average time (in seconds) for answering queries with the reasoners Pellet, BORN, Fact++, HermiT, JFact, TRILL, \trillp and \trillpbdd for the KB of Example~\ref{exp-expl} (\nameref{test:np-3}) with increasing $n$. The cells containing ``--'' mean that the execution  timed out (10 minutes). Bold values indicate the best reasoners for each size.}
    \label{table:res-3-np}  
{
\begin{center}
\begin{tabular}{l|ccccc}
\hline\hline
 & 2 & 4 & 6 & 8 & 10\\\hline
Pellet& 0.348 & 0.579 & 3.069 & -- & --\\
BORN & 3.387 & 3.339 & 3.376 & 3.503 & 4.710\\
Fact++& 0.506 & 1.194 & 3.538 & 13.839 & --\\
HermiT& 1.601 & 7.091 & 34.58 & 262.809 & --\\
JFact& 0.625 & 1.313 & 3.29 & 10.846 & --\\
TRILL & \textbf{0.003} & 0.009 & 0.101 & 4.737 & --\\
\trillp & 0.005 & 0.046 & 6.055 & -- & --\\
\trillpbdd & \textbf{0.003} & \textbf{0.006} & \textbf{0.011} & \textbf{0.019} & \textbf{0.028}\\
\hline\hline
\end{tabular}
\end{center}

%
}
\end{table*}

\subsection{Probabilistic Inference}
Similarly to the previous section, we performed three different tests, two of which extend the first and the third non-probabilistic tests.

\paragraph{Test 4}
\label{test:p-1}
We used the same KBs of the non-probabilistic \nameref{test:np-1} and the systems TRILL, \trillp, \trillpbdd, BUNDLE and BORN. For each KB we added probabilities to 50 of its axioms  randomly chosen. The probability values were learned using EDGE~\cite{RigBelLamZese13-RR13a-IC}, an algorithm for parameter learning from a set of positive and negative examples.
We considered the same queries of \nameref{test:np-1}, but in this case the reasoners are used to compute the probability of the  queries.

Table \ref{table:res-1-p-1} shows the average time in seconds taken by the systems for performing probabilistic inference over different KBs, together with standard deviation.

Comparing the results with  Table \ref{table:res-1-np-2-1}, we see the extra time for the probability computation is negligible, and in some cases these results  are even better: this is due to a time measurement error. It is also worth noting the improvement in terms of performance achieved by BUNDLE with respect to Pellet, on which it is based, obtained by the optimization implemented in BUNDLE's code.
 \trillpbdd proves to be the fastest reasoner.

\begin{table}[htb]
  \caption{Average number of  explanations and average time (in seconds) for computing the probability of queries w.r.t. BioPAX and DBPedia with the reasoners BUNDLE, BORN, TRILL, \trillp and \trillpbdd\ in \nameref{test:p-1}. Each cell contains the running time $\pm$ its standard deviation. ``n.a.'' means not applicable. Bold values highlight the fastest reasoner for each KB.}
    \label{table:res-1-p-1}  

\begin{center}
\begin{tabular}{l|cccc}
\hline\hline
 & BioPAX & DBPedia & Vicodi & BRCA \\
 Avg. N. Expl. & 3.92 & 16.32 & 1.02 & 6.49 \\\hline
BUNDLE & 1.776 $\pm$ 0.078 &  1.374 $\pm$ 0.047 & 1.355 $\pm$ 0.077 & 6.530 $\pm$ 2.863\\
BORN & n.a.  & 3.797 $\pm$ 0.319 & n.a. & n.a.\\
TRILL & 0.139 $\pm$ 0.055 & 0.110 $\pm$ 0.008 & 0.050 $\pm$ 0.023 & 0.794 $\pm$ 0.022 \\
\trillp & 0.114 $\pm$ 0.037 & 0.147 $\pm$ 0.016 & \textbf{0.040 $\pm$ 0.020} & 1.367 $\pm$ 0.038 \\
\trillpbdd & \textbf{0.110 $\pm$ 0.044} & \textbf{0.011 $\pm$ 0.003 } & 0.042 $\pm$ 0.015 & \textbf{0.083 $\pm$ 0.025} \\
\hline\hline
\end{tabular}
\end{center}

\end{table}

%
%

\paragraph{Test 5}
\label{test:p-2}
We used the KB of the non-probabilistic \nameref{test:np-3} where all the axioms
were assigned a random value of probability. As before, 
we increased $n$ from 2 to 10 in steps of 2 and we collected the running time, averaged over 50 executions with timeout set to 10 minutes.
Table \ref{table:res-2-p} shows, for each $n$, the average time in seconds taken by the systems  for computing the probability of the query $Q$. Cells with ``--'' indicate that the timeout occurred.
This test confirms the results of the non-probabilistic $Test\ 3$: \trillpbdd can avoid exponential blow-up. For instance, if we disable the time out, with $n=10$ TRILL took about 17500 seconds and \trillp took more than 24 hours whereas \trillpbdd terminated in less than one second. For $n=200$ TORNADO's running time was about 49 seconds, while for $n=300$ it was about 160 seconds.
Comparing these results with Table~\ref{table:res-3-np}, one can see that most time is spent in finding explanations.  BUNDLE scales better than Pellet since can solve  queries w.r.t. the KB with $n=6$ in less than 10 minutes.
 \trillpbdd again achieves the best performance.
\begin{table*}[htb]
  \caption{Average time (in seconds) for computing the probability of queries with the reasoners BUNDLE, BORN, TRILL, \trillp and \trillpbdd\ in \nameref{test:p-2}.  ``--'' means that the execution  timed out (600 s). Bold values indicate the best reasoners for each size.}
    \label{table:res-2-p}  
{
\begin{center}
\begin{tabular}{l|ccccc}
\hline\hline
 & 2 & 4 & 6 & 8 & 10\\\hline
BUNDLE & 1.307 & 3.116 & 19.860 & 437.118 & --\\
BORN & 4.412 & 4.47 & 4.589 & 4.495 & 4.503\\
TRILL & 0.003 & 0.010 & 0.105 & 4.732 & --\\ 
\trillp & 0.006 & 0.046 & 6.002 & -- & -- \\
\trillpbdd & \textbf{0.002} & \textbf{0.006} & \textbf{0.011} & \textbf{0.018} & \textbf{0.027}\\
\hline\hline
\end{tabular}
\end{center}
}
\end{table*}

\paragraph{Test 6}
\label{test:p-3}
The last test was performed following the approach presented in \cite{DBLP:conf/semweb/KlinovP08} where they investigated the scalability of PRONTO
on versions of BRCA of increasing size. In this test BORN couldn't be used since the expressiveness of BRCA  is higher than that of $\mathcal{EL}$ DL.
We applied PRONTO in two different versions, the version of \cite{DBLP:conf/semweb/KlinovP08} and a second one using a solver for doing LP/MILP programming\footnote{The code of this version of PRONTO was get by personal communication with Pavel Klinov.} (in our tests we used GLPK\footnote{\url{https://en.wikibooks.org/wiki/GLPK}}), presented in \cite{DBLP:conf/cade/KlinovP11}.

To test PRONTO, \citeN{DBLP:conf/semweb/KlinovP08} randomly generated and added an increasing number of conditional constraints in the non-probabilistic KB,  i.e., an increasing number of subclass-of probabilistic axioms. The number of these constraints was varied from 9 to 15, and, for each number, 100 different consistent ontologies were created. 

In this test, we took these KBs and we added an individual to each of them, randomly assigned to each simple class that appears in conditional constraints with probability 0.6. Complex classes contained in the conditional constraints were split into their components, e.g., the complex class \textit{PostmenopausalWomanTakingTestosterone} was divided into \textit{PostmenopausalWoman} and \textit{WomanTakingTestosterone}.
Finally, we ran 100 probabilistic queries of the form $a : C$ where $a$ is the added individual and $C$ is a class randomly selected among those that represent women under increased and lifetime risk such as \textit{WomanUnderLifetimeBRCRisk} and \textit{WomanUnderStronglyIncreasedBRCRisk}.

Figure \ref{fig:res-p-3} shows the execution time averaged over the 100 queries as a function of the number of probabilistic axioms. TRILL, \trillp and BUNDLE behave similarly. PRONTO and PRONTO GLPK show very different behaviors: the first one has an exponential trend while the second one is constant. \trillpbdd outperforms all the  algorithms with a constant trend.

\begin{figure*}
\begin{center}
\includegraphics[scale=0.7]{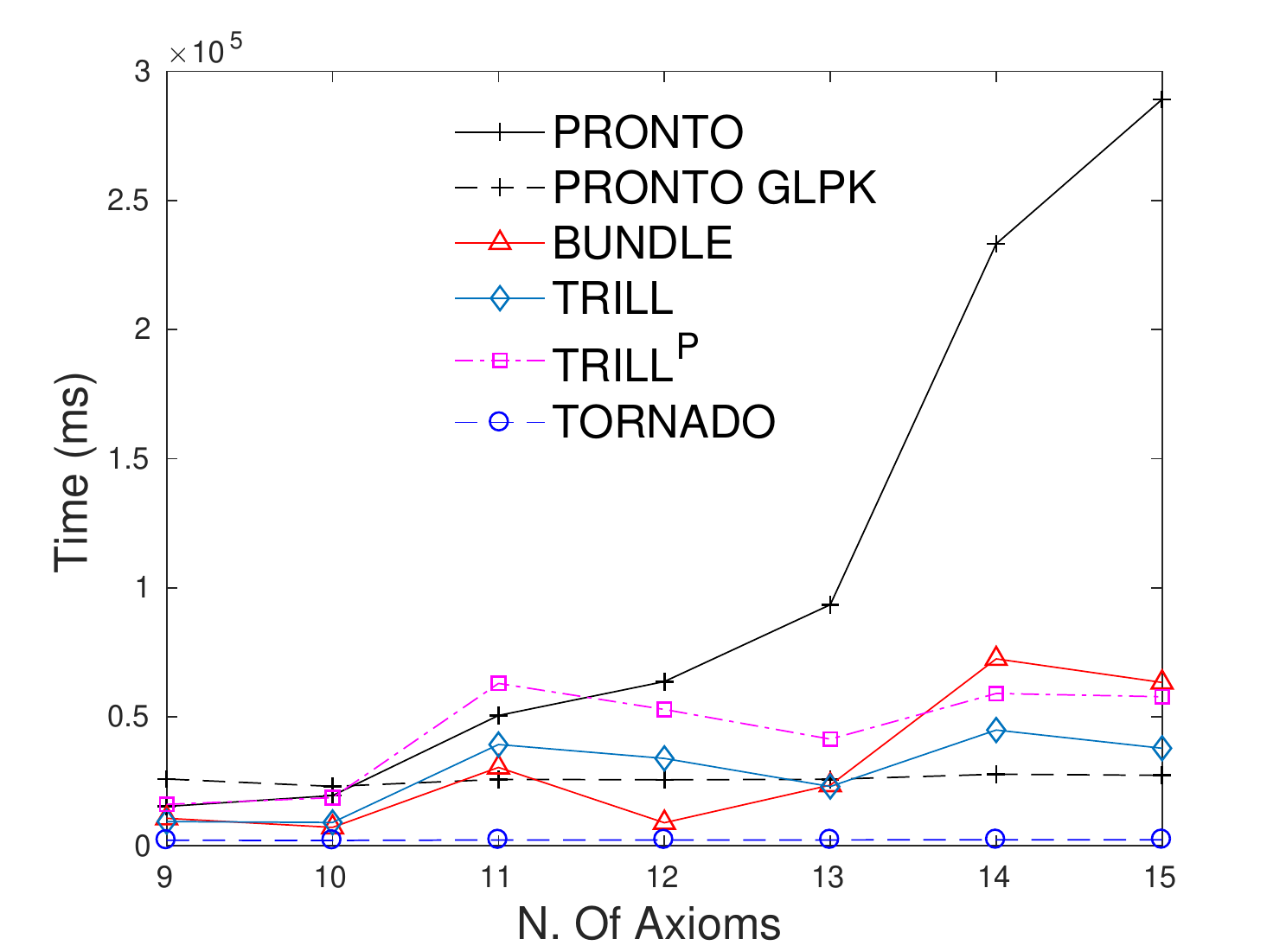}
\end{center}
\caption{Average execution time (ms) for inference with PRONTO, PRONTO GLPK (based on the GLPK LP/MILP solver), BUNDLE, TRILL, \trillp, and \trillpbdd\  on versions of the BRCA KB of increasing size in \nameref{test:p-3}.\label{fig:res-p-3}} 
\end{figure*} 

\subsection{Discussion}
Extensive experimentation shows that, in general, a full Prolog implementation of probabilistic reasoning algorithms for DL can achieve better results than other state-of-the-art probabilistic reasoners such as BORN, BUNDLE, and PRONTO, and thus a Prolog implementation of probabilistic tableau reasoners is feasible and may lead to practical systems. Confirmation of this can also be seen in the performance of BORN, exploiting Probabilistic Logic Programming techniques, which usually performs well. Moreover, the time spent in computing the probability of query is usually a small part of the total execution time, showing that probabilistic reasoners can be used also in non-probabilistic settings. In fact, as shown in non-probabilistic tests,  reasoners implemented in Prolog can achieve better results than other state-of-the-art systems, such as Pellet.
More specifically, constructing BDDs directly during the inference process improves the general performance, as shown by \trillpbdd, avoiding exponential blow-up and, in general, highly improving the scalability of the system. 
From the experimentation, \trillpbdd comes out to be the reasoner with the best performances because its running time is always comparable or better than the best results achieved by the other reasoners.
 However, there are some limitations about the supported expressiveness. In fact, \trillp and \trillpbdd support complete reasoning only for DL \shi, whereas other reasoners, with the exception of BORN, support more expressive DLs. 

%
%

\section{Conclusions}
\label{sec:concl}
In  this paper we  presented the algorithm \trillpbdd for reasoning on DISPONTE KBs that extends and improves the previous systems TRILL and \trillp. \trillpbdd, similarly to \trillp, implements in Prolog the tableau algorithm defined in \cite{DBLP:journals/jar/BaaderP10,DBLP:journals/logcom/BaaderP10}, but instead of building a pinpointing formula and translating it to a BDD in two different phases, it builds the BDD while building the tableau.
The experiments performed show that this can speed up both regular and probabilistic queries over regular or probabilistic KBs


TRILL, \trillp and \trillpbdd can be tested online at \url{http://trill.ml.unife.it/}.
\\\\\textbf{Acknowledgement}
This work was supported by the ``National Group of Computing Science (GNCS-INDAM)''.

\bibliographystyle{acmtrans}
\bibliography{bibliography/journals_short,bibliography/booktitles_long,bibliography/series_long,bibliography/series_springer,bibliography/publishers_long,bibliography/bibl}

\end{document}